\documentclass[11pt,a4paper]{article}
\usepackage[hyperref]{naaclhlt2018}

\usepackage{latexsym}
\usepackage{danudefs}

\usepackage{times}  
\usepackage{helvet}  
\usepackage{courier}  
\usepackage{url}  

\usepackage{color}
\usepackage{natbib}
\usepackage{multirow}
\usepackage{booktabs}
\usepackage{csquotes}
\usepackage{breqn}

\usepackage{esvect}

\usepackage[T1]{fontenc}
 
\usepackage{caption}
\usepackage{subcaption}
\usepackage{graphicx}  
\usepackage{numprint}
\npdecimalsign{.}
\nprounddigits{3}

\aclfinalcopy 

\usepackage{xspace}
\newcommand{\PairDiff}{\textsf{PairDiff}\xspace}

\title{\textit{Why} \PairDiff \textit{works}? -- A Mathematical Analysis of \\Bilinear Relational Compositional Operators for Analogy Detection}

\author{$\text{Huda Hakami}^{\dagger}$ \\ \And
	     $\text{Danushka Bollegala}^{\dagger}$ \\
	     $\text{University of Liverpool}^{\dagger}$ \\
	     $\text{Japan Advanced Institute of Science and Technology}^{\ddagger}$ \\
	     {\tt \{H.A.Hakami,danushka\}@liverpool.ac.uk} \ \ \ {\tt hayashi.kohei@gmail.com}\And
	     $\text{Hayashi Kohei}^{\ddagger}$}

\date{}
\graphicspath{{./Figures/}}
\begin{document}
\maketitle
\begin{abstract}
Representing the semantic relations that exist between two given words (or entities) is an important first step in a wide-range of NLP applications such as analogical reasoning, knowledge base completion and relational information retrieval. A simple, yet surprisingly accurate method for representing a relation between two words is to compute the vector offset (\PairDiff) between their corresponding word embeddings. Despite the empirical success, it remains unclear as to whether \PairDiff is the best operator for obtaining a relational representation from word embeddings. 
We conduct a theoretical analysis of generalised bilinear operators that can be used to measure the $\ell_{2}$ relational distance between two word-pairs.
We show that, if the word embeddings are standardised and uncorrelated,  
such an operator will be independent of bilinear terms, and can be simplified to a linear form, where \PairDiff is a special case.
For numerous word embedding types, we empirically verify the uncorrelation assumption, demonstrating the general applicability of our theoretical result. Moreover, we experimentally discover \PairDiff from the bilinear relation composition operator on several benchmark analogy datasets.
\end{abstract}

\section{Introduction}
\label{sec:intro}


Different types of semantic relations exist between words such as \textsf{HYPERNYMY} between \emph{ostrich} and \emph{bird}, or \textsf{ANTONYMY} between \emph{hot} and \emph{cold}. If we consider entities\footnote{We interchangeably use the terms \emph{word} and \emph{entity} to represent both unigrams as well as a multi-word expressions including named entities.}, we can observe even a richer diversity of relations such as \textsf{FOUNDER-OF} between \emph{Bill Gates} and \emph{Microsoft}, or \textsf{CAPITAL-OF} between \emph{Tokyo} and \emph{Japan}.
Identifying the relations between words and entities is important for various Natural Language Processing (NLP) tasks such as automatic knowledge base completion \citep{socher2013reasoning}, analogical reasoning~\citep{turney2005corpus,bollegala2009relational} and relational information retrieval~\citep{duc2010using}. 
For example, to solve a word analogy problem of the form \enquote{\emph{a} is to \emph{b} as \emph{c} is to \emph{?}}, the relationship between the two words in the pair $(a,b)$ must be correctly identified in order to find candidates $d$ that have similar relations with $c$. 
For example, given the query \enquote{\emph{Bill Gates} is to \emph{Microsoft} as \emph{Steve Jobs} is to \emph{?}}, a relational search engine must retrieve \emph{Apple Inc.} because the \textsf{FOUNDER-OF} relation exists between the first and the second entity pairs.

Two main approaches for creating relation embeddings can be identified in the literature.
In the first approach, from given corpora or knowledge bases, word and relation embeddings are \emph{jointly} learnt such that some objective is optimised~\citep{KALE,Yang:ICLR:2015,Nickel:AAAI:2016,Bordes:NIPS:2013,Rocktaschel:ICLR:2016,AdvSets:2017,ComplexEmb}.
In this approach, word and relation embeddings are considered to be \emph{independent} parameters that must be learnt by the embedding method.
For example, TransE~\citep{Bordes:NIPS:2013} learns the word and relation embeddings such that we can accurately predict relations (links) in a given knowledge base using the learnt word and relation embeddings. Because relations are learnt independently from the words, we refer to methods that are based on this approach as \emph{independent} relational embedding methods.

A second approach for creating relational embeddings is to apply some operator on two word embeddings to \emph{compose} the embedding for the relation that exits between those two words, if any. In contrast to the first approach, we do not have to learn relational embeddings and hence this can be considered as an unsupervised setting,
where the compositional operator is predefined.
A popular operator for composing a relational embedding from two word embeddings is \PairDiff, which is the vector difference (offset) of the word embeddings~\cite{mikolov2013linguistic,levy2014linguistic,Vylomova:ACL:2016,Bollegala:AAAI:2015,blacoe-lapata:2012:EMNLP-CoNLL}.
Specifically, given two words $a$ and $b$ represented by their word embeddings respectively $\vec{a}$ and $\vec{b}$, the relation between $a$ and $b$ is given by $\vec{a} - \vec{b}$ under the \PairDiff operator. \newcite{mikolov2013linguistic} showed that \PairDiff can accurately solve analogy equations such as 
$\vv{\mathbf{king}}-\vv{\mathbf{man}} + \vv{\mathbf{woman}} = \vv{\mathbf{queen}}$, where we have used the top arrows to denote the embeddings of the corresponding words.
\newcite{Bollegala:IJCAI:2015} showed that \PairDiff can be used as a proxy for learning better word embeddings and \newcite{Vylomova:ACL:2016} conducted an extensive empirical comparison of \PairDiff using a dataset containing 16 different relation types.
Besides \PairDiff, concatenation~\citep{Huda:KBS:2017,Yin:ACL:2016}, circular correlation and convolution~\citep{Nickel:AAAI:2016} have been used in prior work for representing the relations between words. 
Because the relation embedding is composed using word embeddings instead of learning as a separate parameter, we refer to methods that are based on this approach as \emph{compositional} relational embedding methods.
Note that in this approach it is implicitly assumed that there exist only a single relation between two words.

In this paper, we focus on the operators that are used in compositional relational embedding methods. 
%
If we assume that the words and relations are represented by vectors embedded in some common space, then the operator  we are seeking must be able to produce a vector representing the relation between two words, given their word embeddings as the only input.
Although there have been different proposals for computing relational embeddings from word embeddings, it remains unclear as to what is the best operator for this task.
The space of operators that can be used to compose relational embeddings is open and vast. 
A space of particular interest from a computational point-of-view is the bilinear operators that can be parametrised using tensors and matrices. 
Specifically, we consider operators that consider pairwise interactions between two word embeddings (second-order terms) and contributions from individual word embeddings towards their relational embedding (first-order terms). 
The optimality of a relational compositional operator can be evaluated, for example, using the expected relational distance/similarity such as $\ell_{2}$ between analogous (positive) vs. nonanalogous (negative) word-pairs.

If we assume that word embeddings are standardised, uncorrelated and word-pairs are i.i.d, then we prove in \S\ref{sec:reps} that bilinear relational compositional operators are independent of bilinear pairwise interactions between the two input word embeddings. 
Moreover, under regularised settings (\S\ref{sec:reg}), the bilinear operator further simplifies to a linear combination of the input embeddings, and the expected loss over positive and negative instances becomes zero.
In \S\ref{sec:cross-corr}, we empirically validate the uncorrelation assumption for different pre-trained word embeddings such as 
the Continuous Bag-of-Words Model (CBOW)~\citep{Milkov:2013}, Skip-Gram with negative sampling (SG)~\citep{Milkov:2013}, Global Vectors (GloVe)~\citep{pennington2014glove}, word embeddings created using Latent Semantic Analysis (LSA)~\citep{LSA}, Sparse Coding (HSC)~\citep{faruqui-EtAl:2015:ACL-IJCNLP,Yogatama:ICML:2015}, and Latent Dirichlet Allocation (LDA)~\citep{Blei:JMLR:2003}. 
This empirical evidence implies that our theoretical analysis is applicable to relational representations composed from a wide-range of word embedding learning methods.
Moreover, our experimental results show that a bilinear operator reaches its optimal performance in two different word-analogy benchmark datasets, when it satisfies the requirements of the \PairDiff operator. 
We hope that our theoretical analysis will expand the understanding of relational embedding methods, and inspire future research on accurate relational embedding methods using word embeddings as the input.

\section{Related Work}
\label{sec:related}

As already mentioned in \S\ref{sec:intro}, methods for representing a relation between two words can be broadly categorised into two groups depending on whether the relational embeddings are learnt \emph{independently} of the word embeddings, or they are \emph{composed} from the word embeddings, in which case the relational embeddings fully depend on the input word embeddings. Next, we briefly overview the different methods that fall under each category. For a detailed survey of relation embedding methods see \cite{Nickel:2016}.

Given a knowledge base where an entity $h$ is linked to an entity $t$ by a relation $r$,  the TransE model~\citep{Bordes:NIPS:2013} scores the tuple $(h,t,r)$ by
the $\ell_{1}$ or $\ell_{2}$ norm of the vector $(\vec{h} + \vec{r} - \vec{t})$. 
\cite{Nickerl:ICML:2011} proposed RESCAL, which uses $\vec{h}\T\mat{M}_{r}\vec{t}$ as the scoring function, where $\mat{M}_{r}$ is a matrix embedding of the relation $r$. Similar to RESCAL, Neural Tensor Network~\citep{socher2013reasoning} also models a relation by a matrix. However, compared to vector embeddings of relations, matrix embeddings increase the number of parameters to be estimated, resulting in an increase in computational time/space and likely to overfit.
To overcome these limitations, DistMult~\citep{Yang:ICLR:2015} models relations by vectors and use elementwise multilinear dot product $\vec{r} \odot \vec{h} \odot \vec{t}$. Unfortunately, DistMult cannot capture directionality of a relation. Complex Embeddings~\citep{ComplexEmb} overcome this limitation of DistMult by using complex embeddings and defining the score to be the real part of $\vec{r} \odot \vec{h} \odot \bar{\vec{t}}$, where $\bar{\vec{t}}$ denotes the complex conjugate of $\vec{t}$. 

The observation made by \citet{mikolov2013linguistic} that the relation between two words can be represented by the difference between their word embeddings sparked a renewed interest in methods that compose relational embeddings using word embeddings. 
Word analogy datasets such as Google dataset~\citep{mikolov2013linguistic}, SemEval 2012 Task2 dataset~\citep{Jurgens:SEM:2012}, BATS~\citep{drozd-gladkova-matsuoka:2016:COLING} etc. have established as benchmarks for evaluating word embedding learning methods.

Different methods have been proposed to measure the similarity between the relations that exist between two given word pairs such as \textsf{CosMult}, \textsf{CosAdd} and \PairDiff~\citep{levy2014linguistic,Bollegala:IJCAI:2015}.
\citet{Vylomova:ACL:2016} studied as to what extent the vectors generated using simple \PairDiff encode different relation types. 
Under supervised classification settings, they conclude that \PairDiff can cover a wide range of semantic relation types. 
Holographic embeddings proposed by \citet{Nickel:AAAI:2016} use circular convolution to mix the embeddings of two words to create an embedding for the relation that exist between those words.
It can be showed that circular correlation is indeed an elementwise product in the Fourier space and is mathematically equivalent to complex embeddings~\citep{Hayashi:ACL:2017}.

Although \PairDiff operator has been widely used in prior work for computing relation embeddings from word embeddings, to the best of our knowledge, no theoretical analysis has been conducted so far explaining why and under what conditions \PairDiff is optimal, which is the focus of this paper.


\setcounter{secnumdepth}{2}

\section{Bilinear Relation Representations}
\label{sec:reps}

Let us consider the problem of representing the semantic relation $\vec{r}(\vec{h}, \vec{t})$ between two given words $h$ and $t$.
We assume that $h$ and $t$ are already represented in some $d$-dimensional space  respectively by their word embeddings $\vec{h}, \vec{t} \in \R^{d}$. The relation between two words can be represented using different linear algebraic structures. 
Two popular alternatives are vectors~\citep{Nickel:AAAI:2016,Bordes:NIPS:2013,AdvSets:2017,ComplexEmb} and matrices~\citep{socher2013reasoning,Bollegala:AAAI:2015}. 
Vector representations are preferred over matrix representations because of the smaller number of parameters to be learnt~\citep{Nickel:2016}.

Let us assume that the relation $r$ is represented by a vector $\vec{r} \in \R^{\delta}$ in some $\delta$-dimensional space.
Therefore, we can write $\vec{r}(\vec{h}, \vec{t})$ as a function that takes two vectors (corresponding to the embeddings of the two words) as the input and returns a single vector (representing the relation between the two words) as given in \eqref{eq:relrep}.
\begin{align}
 \label{eq:relrep}
\myfunc{\vec{r}}{\R^{d} \times \R^{d}}{\R^{\delta}}
\end{align}
Having both words and relations represented in the same $\delta = d$ dimensional space is useful for performing linear algebraic operations using those representations in that space. For example, in TransE~\citep{Bordes:NIPS:2013}, the strength of a relation $r$ that exists between two words $h$ and $t$ is computed as the $\ell_{1,2}$ norm of the vector $(\vec{h} + \vec{r} - \vec{t})$ using the word and relation embeddings. Such direct comparisons between word and relation embeddings would not be possible if words and relations were not embedded in the same vector space. 
If $\delta\!<\!d$, we can first project word embeddings to a lower $\delta$-dimensional space using some dimensionality reduction method such as SVD, whereas if $\delta\!>\!\!d$ we can learn higher $\delta$-dimensional overcomplete word representations~\cite{faruqui-EtAl:2015:ACL-IJCNLP} from the original $d$-dimensional word embeddings. Therefore, we will limit our theoretical analysis to the $\delta=d$ case for ease of description.

Different functions can be used as $\vec{r}(\vec{h}, \vec{t})$ that satisfy the domain and range requirements specified by \eqref{eq:relrep}.
If we limit ourselves to bilinear functions, the most general functional form is given by \eqref{eq:rel}.
\begin{align}
 \label{eq:rel}
 \vec{r}(\vec{h}, \vec{t}) = \vec{h}\T \underline{\mat{A}} \vec{t} + \mat{P}\vec{h} + \mat{Q}\vec{t}
\end{align}
Here, $\underline{\mat{A}} \in \R^{d \times d \times d}$ is a 3-way tensor in which each slice is a $d \times d$ real matrix.
Let us denote the $k$-th slice of $\underline{\mat{A}}$ by $\mat{A}^{(k)}$ and its $(i,j)$ element by $A_{ij}^{(k)}$.
The first term in \eqref{eq:rel} corresponds to the pairwise interactions between $\vec{h}$ and $\vec{t}$. 
$\mat{P}, \mat{Q} \in \R^{d \times d}$ are the nonsingular\footnote{If the projection matrix is nonsingular, then the inverse projection exists, which preserves the dimensionality of the embedding space.} projection matrices involving first-order contributions respectively of $\vec{h}$ and $\vec{t}$ towards $\vec{r}$.

Let us consider the problem of learning the simplest bilinear functional form according to \eqref{eq:rel} from a given dataset of 
analogous word-pairs $\cD_{+} = \{((h,t), (h',t'))\}$. Specifically, we would like to learn the parameters $\underline{\mat{A}}$, $\mat{P}$ and $\mat{Q}$ such that some distance (loss) between analogous word-pairs is minimised.
As a concrete example of a distance function, let us consider the popularly used Euclidean distance\footnote{For $\ell_{2}$ normalised vectors, their Euclidean distance is a monotonously decreasing function of their cosine similarity.} ($\ell_{2}$ loss) for two word pairs given by \eqref{eq:loss}.
\begin{align}
 \label{eq:loss}
 J((h,t), (h',t')) = \norm{\vec{r}(\vec{h}, \vec{t}) - \vec{r}(\vec{h'}, \vec{t'})}_{2}^{2}
\end{align}

If we were provided only analogous word-pairs (i.e. positive examples), then this task could be trivially achieved by setting all parameters to zero.
However, such a trivial solution would not generalise to unseen test data. Therefore, in addition to $\cD_{+}$ we would require a set of non-analogous word-pairs $\cD_{-}$ as negative examples. Such negative examples are often generated in prior work by randomly corrupting positive relational tuples~\citep{Nickel:AAAI:2016,Bordes:NIPS:2013,ComplexEmb} or by training an adversarial generator~\citep{AdvSets:2017}. 

The total loss $J$ over both positive and negative training data can be written as follows:

{\small
\begin{align}
 \label{eq:total-loss}
J = & \sum_{((h,t), (h',t')) \in \cD_{+}}  \norm{\vec{r}(\vec{h}, \vec{t}) - \vec{r}(\vec{h'}, \vec{t'})}_{2}^{2} \nonumber \\
 & - \sum_{((h,t), (h',t')) \in \cD_{-}}  \norm{\vec{r}(\vec{h}, \vec{t}) - \vec{r}(\vec{h'}, \vec{t'})}_{2}^{2}
\end{align}
}

Assuming that the training word-pairs are randomly sampled from $\cD_{+}$ and $\cD_{-}$ according to two distributions respectively
$p_{+}$ and $p_{-}$, we can compute the total expected loss, $\Ep_{p}[J]$, as follows:
\begin{align}
\label{eq:exp-loss}
 \Ep_{p}[J] =& \Ep_{p_{+}}\left[\norm{\vec{r}(\vec{h}, \vec{t}) - \vec{r}(\vec{h'}, \vec{t'})}_{2}^{2} \right] - \nonumber \\
 &\Ep_{p_{-}}\left[\norm{\vec{r}(\vec{h}, \vec{t}) - \vec{r}(\vec{h'}, \vec{t'})}_{2}^{2} \right] 
\end{align}

We make the following assumptions to further analyse the properties of relational embeddings.
\begin{description}
\item[Uncorrelation:]
The correlation between any two distinct dimensions of a word embedding is zero. 
One might think that the uncorrelation of word embedding dimensions to be a strong assumption, but we later show its validity empirically in~\S\ref{sec:cross-corr} for a wide range of word embeddings.

\item[Standerdisation:]
Word embeddings are standerdised to zero mean and unit variance. This is a linear transformation in the word embedding space and does not
affect the topology of the embedding space. In particular, translating word embeddings such that they have a zero mean has shown to improve performance in similarity tasks~\cite{All-but-Top}.

\item[Relational Independence] Word pairs in the training data are assumed to be i.i.d. For example, whether a particular semantic relation $r$ exists between $h$ and $t$, is assumed to be independent of any other relation $r'$ that exists between $h'$ and $t'$ in a different pair.
\end{description}

For relation representations given by \eqref{eq:rel}, \autoref{th:pd}  holds:
\begin{theorem}
\label{th:pd}
Consider  the bilinear relational embedding defined by \eqref{eq:rel} computed using uncorrelated word embeddings.
If the word embeddings are standerdised, then the expected loss given by \eqref{eq:exp-loss} over a relationally independent set of word pairs is independent of $\underline{\mat{A}}$.
\end{theorem}

\begin{proof}

Let us consider the bilinear term in \eqref{eq:rel}, because $i$ and $j (\neq i)$ dimensions of word embeddings are uncorrelated by the assumption (i.e. $\textrm{corr}(u_{i}, u_{j}) = 0$), from the definition of correlation we have,
\begin{align}
 \textrm{corr}(u_{i}, u_{j}) &= \Ep[u_{i}u_{j}] - \Ep[u_{i}]\Ep[u_{j}] = 0 \\ 
 \label{eq:prod}
 \Ep[u_{i}u_{j}] &= \Ep[u_{i}]\Ep[u_{j}] .
\end{align}
Moreover, from the standerdisation assumption we have, $\Ep[u_{i}] = 0, \ \  \forall_{i = 1 \ldots n}$. From \eqref{eq:prod} it follows that:
\begin{align}
\label{eq:zero}
\Ep[u_{i}u_{j}] = 0
\end{align}
for $i \neq j$ dimensions.

We will next show that \eqref{eq:exp-loss} is independent of $\underline{\mat{A}}$.
For this purpose, let us consider the $\Ep_{p_{+}}$ term first and 
write the $k$-th dimension of $\vec{r}(\vec{h}, \vec{t})$ using $\mat{A}^{(k)}$, $\mat{P}$ and $\mat{Q}$ as follows:
\begin{align}
\label{eq:k}
 \sum_{i,j} \left( A^{(k)}_{ij}h_{i}t_{j}\right) + \sum_{n}P_{kn}h_{n} + \sum_{n}Q_{kn}t_{n}
\end{align}
Plugging \eqref{eq:k} in \eqref{eq:exp-loss} and computing the loss over all positive training instances we get,

{\small
\begin{align}
\label{eq:sqd}
& \Ep_{p_{+}}[ \sum_{k} ( \sum_{i,j} \left( A^{(k)}_{ij}(h_{i}t_{j} - h'_{i}t'_{j})\right)+ \nonumber \\
&  \sum_{n}P_{kn}(h_{n} - h'_{n}) + \sum_{n}Q_{kn}(t_{n} - t'_{n}) )^{2} ]
\end{align}
}%

Terms that involve only elements in $\mat{A}^{(k)}$ take the form:
{\small
\begin{align}
  \sum_{i,j}&\sum_{l,m}\Ep_{p_{+}} \left[ A^{(k)}_{ij}A^{(k)}_{lm}(h_{i}t_{j} - h'_{i}t'_{j})(h_{l}t_{m} - h'_{l}t'_{m}) \right] \nonumber \\ 
 =& \sum_{i,j}\sum_{l,m} A^{(k)}_{ij}A^{(k)}_{lm} ( \Ep_{p_{+}}[h_{i}t_{j}h_{l}t_{m}] - \Ep_{p_{+}}[h_{i}t_{j}h'_{l}t'_{m}] -  \nonumber \\
   &\Ep_{p_{+}}[h'_{i}t'_{j}h_{l}t_{m}] + \Ep_{p_{+}}[h'_{i}t'_{j}h'_{l}t'_{m}] ) \label{eq:four}
\end{align}
}%

In cases where $i \neq j$ and $l \neq m$, each of the four expectations in \eqref{eq:four} contains the product of different dimensionalities, which is zero from \eqref{eq:zero}. 
For $i = j = l = m$ case we have,

{\small
\begin{align}
\label{eq:same}
 {A^{(k)}_{ij}}^{2}( \Ep_{p_{+}}[h_{i}^{2}t_{i}^{2}] - 2\Ep_{p_{+}}[h_{i} t_{i}h'_{i} t'_{i}]  + \Ep_{p_{+}}[{h'_{i}}^{2} {t'_{i}}^{2}] )
 \end{align}
 }%
From the relational independence we have $\Ep_{p_{+}}[h_{i}t_{i}h'_{i}t'_{i}] = \Ep_{p_{+}}[h_{i}t_{i}]\Ep_{p_{+}}[h'_{i}t'_{i}]$.
Moreover, because the word embeddings are assumed to be standerdised to unit variance we have 
$\Ep_{p_{+}}[h_{i}t_{i}] = \Ep_{p_{+}}[h'_{i}t'_{i}] = 1$ and $\Ep_{p_{+}}[h^{2}_{i}t^{2}_{i}] = \Ep_{p_{+}}[{h'}^{2}_{i} {t'}^{2}_{i}] = 1$.
Therefore, \eqref{eq:same} evaluates to zero and none of the terms arising purely from $\underline{\mat{A}}$ will remain 
in the expected loss over positive examples.

Next, lets consider the $A^{(k)}_{ij}P_{kn}$ terms in the expansion of \eqref{eq:sqd} given by,
\begin{align}
\label{eq:A-p}
2\sum_{i,j}\sum_{n}A^{(k)}_{ij} P_{kn} (h_{i}t_{j} - h'_{i}t'_{j})(h_{n} - h'_{n}) .
\end{align}
Taking the expectation of \eqref{eq:A-p} w.r.t. $p_{+}$ we get,
{\small
\begin{align}
\label{eq:A-p-exp}
&  2\sum_{i,j}\sum_{n} A^{(k)}_{ij} P_{kn}  ( \Ep_{p_{+}}[h_{i}t_{j}h_{n}] - \Ep_{p_{+}}[h_{i}t_{j}h'_{n}] -  \nonumber \\ 
&  \Ep_{p_{+}}[h'_{i}t'_{j}h_{n}] + \Ep_{p_{+}}[h'_{i}t'_{j}h'_{n}]) .
\end{align}}
Likewise, from the uncorrelation assumption and relational independence it follows that all the expectations in \eqref{eq:A-p-exp} are zero.
A similar argument can be used to show that terms that involve $A^{(k)}_{ij}Q_{kn}$ disappear from \eqref{eq:sqd}.
Therefore, $\underline{\mat{A}}$ does not play any part in the expected loss over positive examples.
Similarly, we can show that $\underline{\mat{A}}$ is independent of the expected loss over negative examples.
Therefore, from \eqref{eq:exp-loss} we see that the expected loss over the entire training dataset is independent of $\underline{\mat{A}}$.

\end{proof}

\subsection{Regularised $\ell_{2}$ loss}
\label{sec:reg}

As a special case, if we attempt to minimise the expected loss under some regularisation on $\underline{\mat{A}}$ such as the Frobenius norm regularisation, then this can be achieved by sending $\underline{\mat{A}}$ to zero tensor because according to \autoref{th:pd} \eqref{eq:rel} is independent from $\underline{\mat{A}}$.

With $\underline{\mat{A}} = \underline{\mat{0}}$, the relation between $h$ and $t$ can be simplified to:
\begin{align}
\label{eq:simp}
 \vec{r}(\vec{h}, \vec{t}) = \mat{P}\vec{h} + \mat{Q}\vec{t} 
\end{align}

Then the expected loss over the positive instances is given by \eqref{eq:pposexp}.
{\small
\begin{align}
\label{eq:pposexp}
&\Ep_{p_{+}}[\norm{\mat{P}(\vec{h}-\vec{h'}) + \mat{Q}(\vec{t}-\vec{t'})}_{2}^{2}]=\nonumber \\
& \Ep_{p_{+}}[(\vec{h}-\vec{h'})\T\mat{P}\T\mat{P}(\vec{h}-\vec{h'})] + \Ep_{p_{+}}[(\vec{h}-\vec{h'})\T\mat{P}\T\mat{Q}(\vec{t}-\vec{t'})] + \nonumber \\
& \Ep_{p_{+}}[(\vec{t}-\vec{t'})\T\mat{Q}\T\mat{P}(\vec{h}-\vec{h'})] + \Ep_{p_{+}}[(\vec{t}-\vec{t'})\T\mat{Q}\T\mat{Q}(\vec{t}-\vec{t'})]
\end{align}
 }%
 The second expectation term in RHS of \eqref{eq:pposexp} can be computed as follows:
{\small
\begin{align}
& \Ep_{p_{+}}[(\vec{h}-\vec{h'})\T\mat{P}\T\mat{Q}(\vec{t}-\vec{t'})] \nonumber \\
& =\sum_{i,j} {(\mat{P}\T\mat{Q})}_{ij} \Ep_{p_{+}}[(h_{i} - h'_{i})(t_{j} - t'_{j})] \nonumber \\
& = \sum_{i,j}{(\mat{P}\T\mat{Q})}_{ij} \left( \Ep_{p_{+}}[h_{i}t_{j}] - \Ep_{p_{+}}[h_{i}t'_{j}] - \Ep_{p_{+}}[h'_{i}t_{j}] + \Ep_{p_{+}}[h'_{i}t'_{j}] \right)
\label{eq:pposexp2}
 \end{align}
}%
When $i \neq j$, each of the four expectations in the RHS of \eqref{eq:pposexp2} are zero from the uncorrelation assumption. When $i = j$, each term will be equal to one from the standeridisation assumption (unit variance) and cancel each other out.
A similar argument can be used to show that the third expectation term in the RHS of \eqref{eq:pposexp} vanishes.

Now lets consider the first expectation term in the RHS of \eqref{eq:pposexp}, which can be computed as follows:

{\small
\begin{align}
 &\Ep_{p_{+}}[(\vec{h}-\vec{h'})\T\mat{P}\T\mat{P}(\vec{h}-\vec{h'})] \nonumber \\
 &=\sum_{i,j} {(\mat{P}\T\mat{P})}_{ij} \Ep_{p_{+}}[(h_{i} - h'_{i})(h_{j} - h'_{j})] \nonumber \\
& = \sum_{i,j}{(\mat{P}\T\mat{P})}_{ij} ( \Ep_{p_{+}}[h_{i}h_{j}] - \Ep_{p_{+}}[h_{i}h'_{j}] \nonumber \\ &
                                                          \qquad \qquad \qquad -\Ep_{p_{+}}[h'_{i}h_{j}] + \Ep_{p_{+}}[h'_{i}h'_{j}] ) 
\label{eq:pposexp1}
 \end{align}
}%
When $i \neq j$, it follows from the uncorrelation assumption that each of the four expectation terms in the RHS of 
\eqref{eq:pposexp1} will be zero. For $i=j$ case we have,

{\small
\begin{align}
 & \sum_{i,j}{(\mat{P}\T\mat{P})}_{ii} \left( \Ep_{p_{+}}[h_{i}^{2}] - 2\Ep_{p_{+}}[h_{i}h'_{i}] + \Ep_{p_{+}}[{h'_{i}}^{2}] \right) \nonumber \\
 &=2 \sum_{i,j}{(\mat{P}\T\mat{P})}_{ii} 
\label{eq:pposexp2}
 \end{align}
}%
Note that from the relational independence between $h$ and $h'$ we have
$\Ep_{p_{+}}[h_{i}h'_{i}] = \Ep_{p_{+}}[h_{i}]\Ep_{p_{+}}[h'_{i}]$. From the standerdidation (zero mean) assumption this term is zero. 
On the other hand $\Ep_{p_{+}}[h_{i}^{2}] = \Ep_{p_{+}}[{h'_{i}}^{2}] = 1$ from the standerdidation (unit variance) assumption, which gives the result in \eqref{eq:pposexp2}.

Similarly, the fourth expectation term in the RHS of \eqref{eq:pposexp}
 evaluates to $2\sum_{i,j} {(\mat{Q}\T\mat{Q})}_{ii}$,
which shows that \eqref{eq:pposexp} evaluates to 
$2\sum_{i,j} \left( {(\mat{P}\T\mat{P})}_{ii} + {(\mat{Q}\T\mat{Q})}_{ii}\right)$.
Note that this is independent of the positive instances and will be equal to the expected loss over negative instances,
which gives $\Ep_{p}[J] = 0$ for the relational embedding given by \eqref{eq:simp}.

It is interesting to note that \PairDiff is a special case of \eqref{eq:simp}, where $\mat{P} = \mat{I}$ and
$\mat{Q} = -\mat{I}$. In the general case where word embeddings are nonstanderdised to unit variance,
we can set $\mat{P}$ to be the diagonal matrix where $\mat{P}_{ii} = 1/\sigma_{i}$, where $\sigma_{i}$ is the variance of the $i$-th dimension of the word embedding space, to enforce standerdisation.
Considering that $\mat{P}, \mat{Q}$ are parameters of the relational embedding, this is analogous to \emph{batch normalisation}~\cite{batch-norm}, where the appropriate parameters for the normalisation are learnt during training.


\setcounter{secnumdepth}{2}

\section{Experimental Results}
\label{sec:exp}

\subsection{Corss-dimensional Correlations}
\label{sec:cross-corr}

A key assumption in our theoretical analysis is the uncorrelations between different dimensions in word embeddings.
Here, we empirically verify the uncorrelation assumption  for different input word embeddings. 
For this purpose, we create SG, CBOW and GloVe embeddings from the ukWaC corpus\footnote{\url{http://wacky.sslmit.unibo.it/doku.php?id=corpora}}.
We use a context window of 5 tokens and select words that occur at least 6 times in the corpus.
We use the publicly available implementations for those methods by the original authors and set the parameters to the recommended values in \citep{Levy:TACL:2015} to create $50$-dimensional word embeddings.
As a representative of counting-based word embeddings, we create a word co-occurrence matrix weighted by the positive pointwise mutual information (PPMI) and apply singular value decomposition (SVD) to obtain $50$-dimensional embeddings, which we refer to as the Latent Semantic Analysis (LSA) embeddings. 

We use Latent Dirichlet Allocation (LDA)~\citep{blei2003latent} to create a topic model, and represent each word by its distribution over the set of topics. Ideally, each topic will capture some semantic category and the topic distribution provides a semantic representation for a word. 
We use gensim\footnote{\url{https://radimrehurek.com/gensim/wiki.html}} to extract $50$ topics from a 2017 January dump of English Wikipedia. 
In contrast to the above-mentioned word embeddings, which are dense and flat structured, we used Hierarchical Sparse Coding\footnote{\url{http://www.cs.cmu.edu/~ark/dyogatam/wordvecs/}} (HSC)~\citep{Yogatama:ICML:2015} to produce sparse and hierarchical word embeddings.
\begin{figure}[t!]
  \centering
  \includegraphics[width=3.2in]{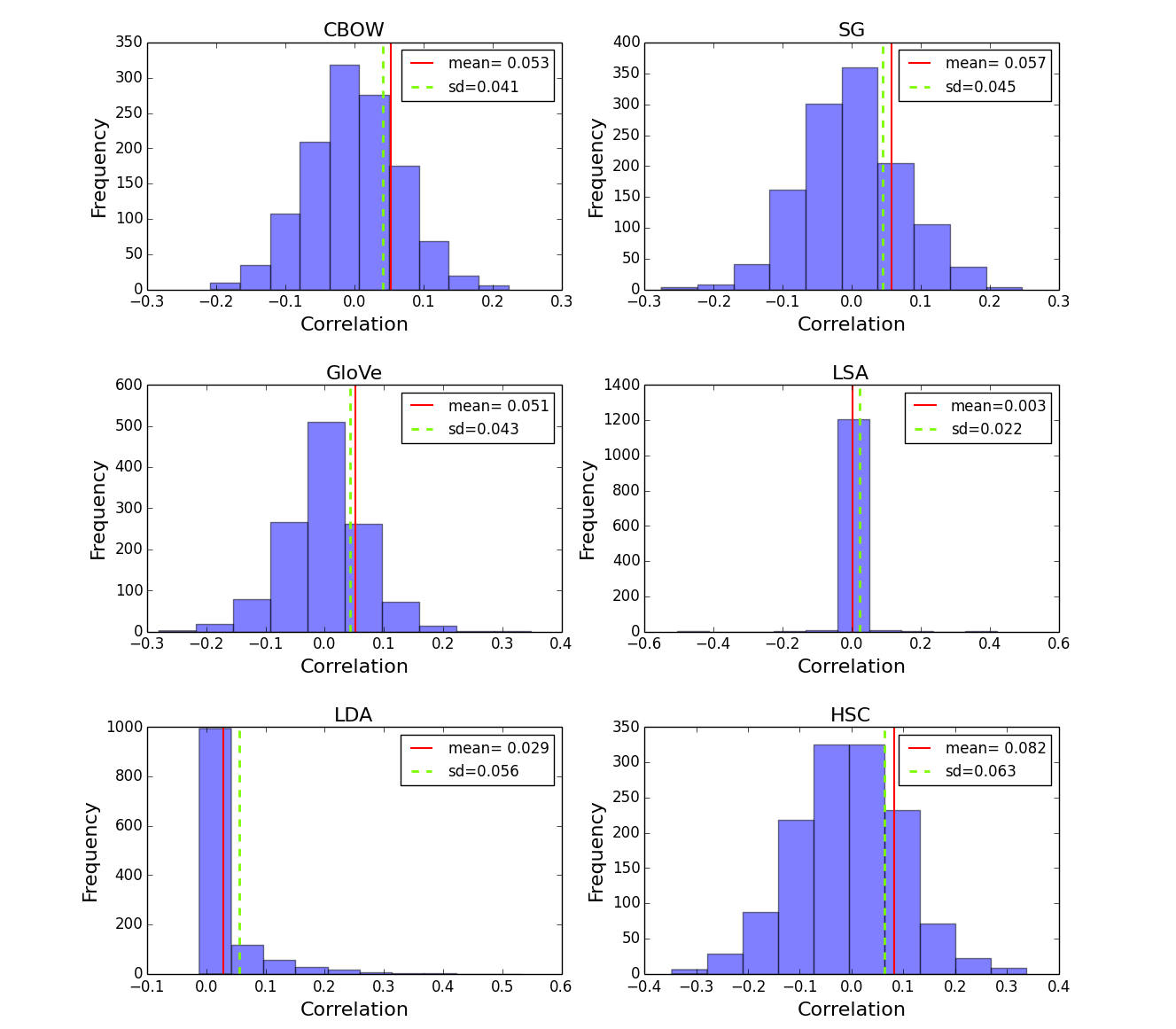}
  \caption{Cross-dimensional correlations for six word embeddings.}
  \label{fig:corr}
\end{figure}

Given a word embedding matrix $\mat{W} \in \R^{m \times d}$, where each row correspond to the $d$-dimensional embedding of a word in a vocabulary containing $m$ words, we compute a correlation matrix $\mat{C} \in \R^{d \times d}$, where the $(i,j)$ element, $C_{ij}$, denotes the Pearson correlation coefficient between the $i$-th and $j$-th dimensions in the word embeddings over the $m$ words.
By construction $C_{ii} = 1$ and the histograms of the cross-dimensional correlations ($i \neq j$) are shown in \autoref{fig:corr} for $50$ dimensional word embeddings obtained from the six methods described above. The mean of the absolute pairwise correlations for each embedding type and the standard deviation (sd) are indicated in the figure. 

From \autoref{fig:corr}, irrespective of the word embedding learning method used, we see that cross-dimensional correlations are distributed in a narrow range with an almost zero mean. 
This result empirically validates the uncorrelation assumption we used in our theoretical analysis.
Moreover, this result indicates that \autoref{th:pd} can be applied to a wide-range of existing word embeddings.

\subsection{Learning Relation Representations}
\label{sec:learn}

Our theoretical analysis in \S\ref{sec:reps} claims that the performance of the bilinear relational embedding is independent of the tensor operator $\underline{\mat{A}}$.
To empirically verify this claim, we conduct the following experiment.
For this purpose, we use the BATS dataset~\citep{gladkova2016analogy} that contains of 40 semantic and syntactic relation types\footnote{\url{http://vsm.blackbird.pw/bats}}, and generate positive examples by pairing word-pairs that have the same relation types. 
Approximately each relation type has 1,225 word-pairs, which enables us to generate a total of 48k positive training instances (analogous word-pairs) of the form $((h,t),(h',t'))$. 
For each pair $(h,t)$ related by a relation $r$, we randomly select pairs $(h',t')$ with a different relation type $r'$, according to the $\ell_{2}$ distance between the two pairs to create negative (nonanalogous) instances.\footnote{10 negative instances are generated from each word-pair in our experiments.}
We collectively refer both positive and negative training instances as the \emph{training} dataset.

Using the $d = 50$ dimensional word embeddings from CBOW, SG, GloVe, LSA, LDA, and HSC methods created in \S\ref{sec:cross-corr},
we learn relational embeddings according to \eqref{eq:rel} by minimising the $\ell_{2}$ loss, \eqref{eq:total-loss}.
To avoid overfitting, we perform $\ell_{2}$ regularisation on $\underline{\mat{A}}$, $\mat{P}$ and $\mat{Q}$ are regularised to diagonal matrices $p\mat{I}$ and $q\mat{I}$, for $p,q \in \R$.
We initialise all parameters by uniformly sampling from $[-1,+1]$ and use AdaGrad~\citep{duchi2011adaptive} with initial learning rate set to 0.01. 

\begin{figure}[t]
\centering

    \begin{subfigure}{0.23\textwidth}
    \includegraphics[width=\textwidth]{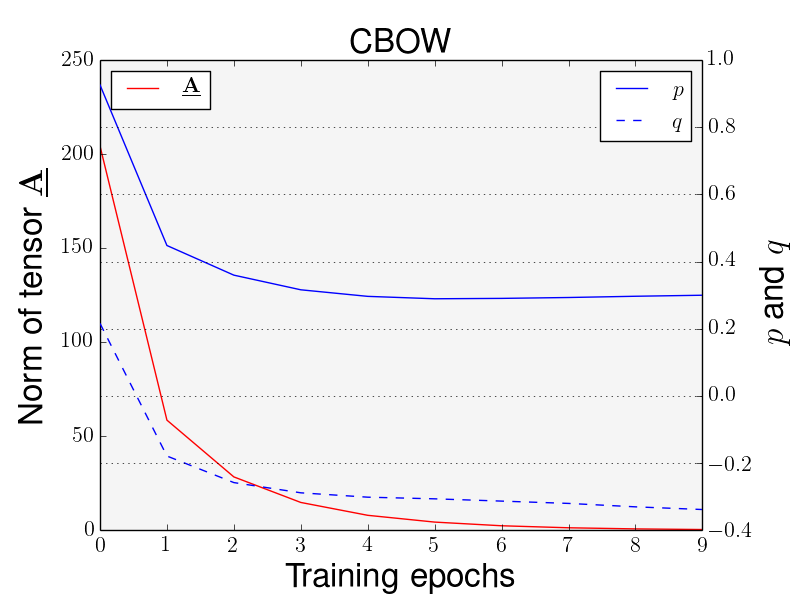}
    \end{subfigure}
    \begin{subfigure}{0.23\textwidth}
    \includegraphics[width=\textwidth]{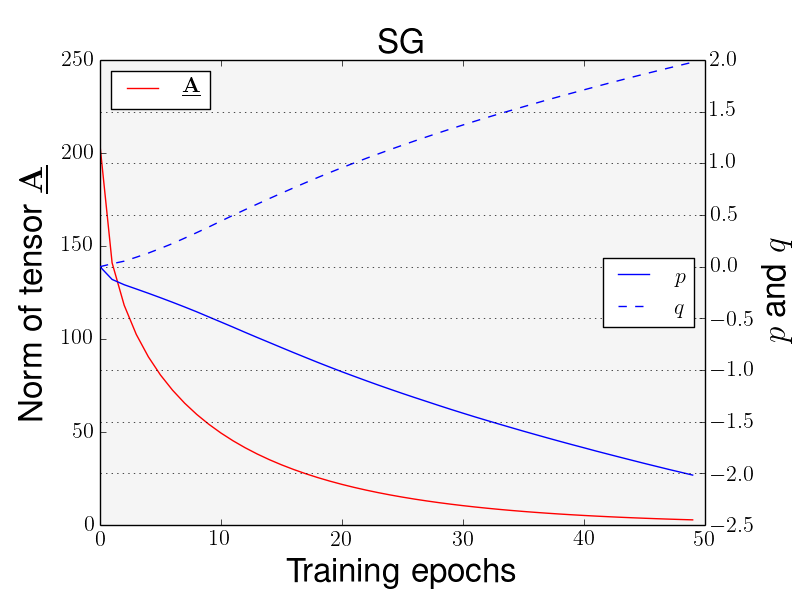}
    \end{subfigure}
    \begin{subfigure}{0.23\textwidth}
    \includegraphics[width=\textwidth]{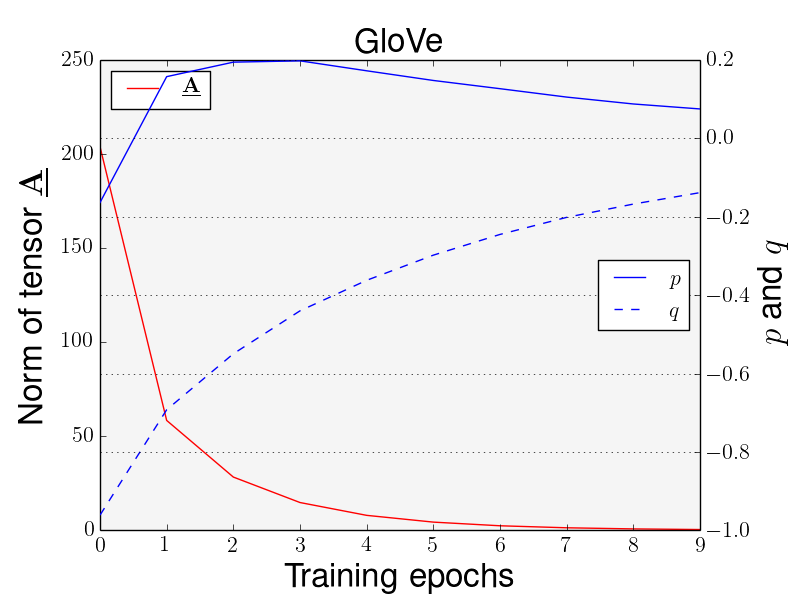}
   \end{subfigure}
   \begin{subfigure}{0.23\textwidth}
    \includegraphics[width=\textwidth]{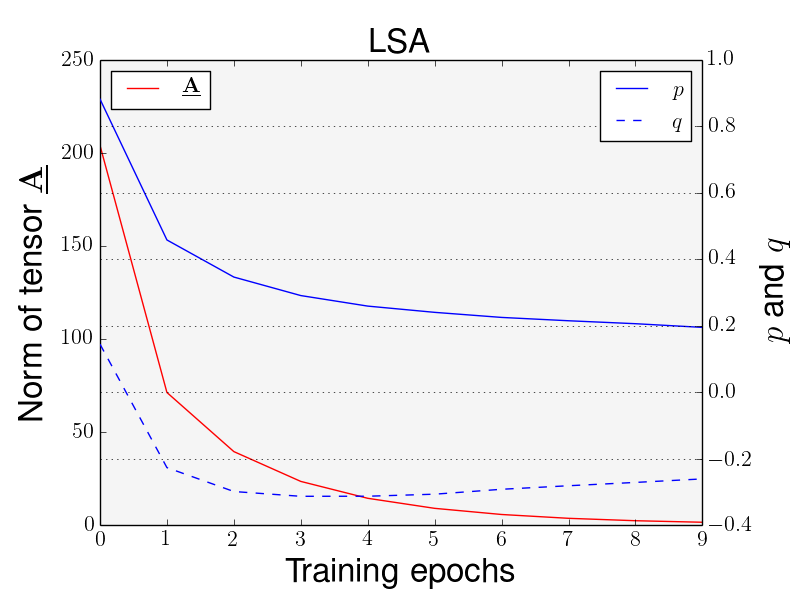}
   \end{subfigure}
   \begin{subfigure}{0.23\textwidth}
   \includegraphics[width=\textwidth]{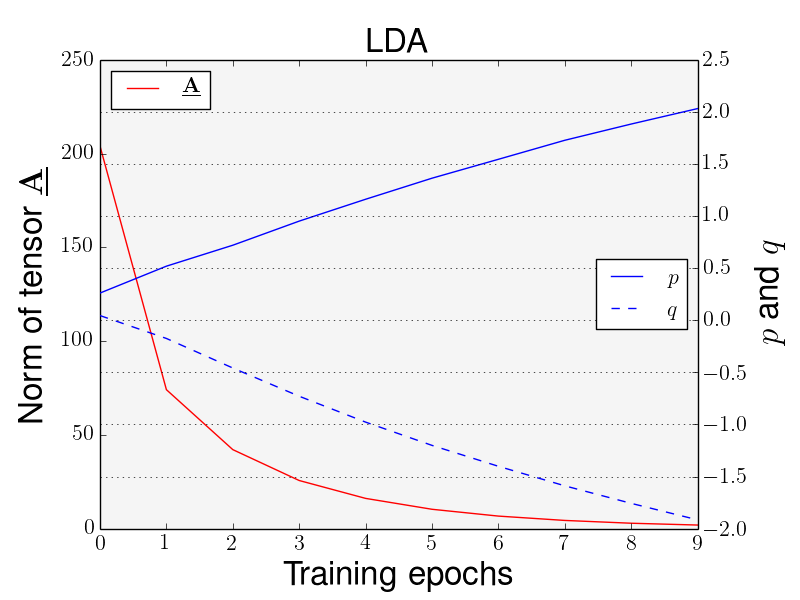}
   \end{subfigure}
   \begin{subfigure}{0.23\textwidth}
   \includegraphics[width=\textwidth]{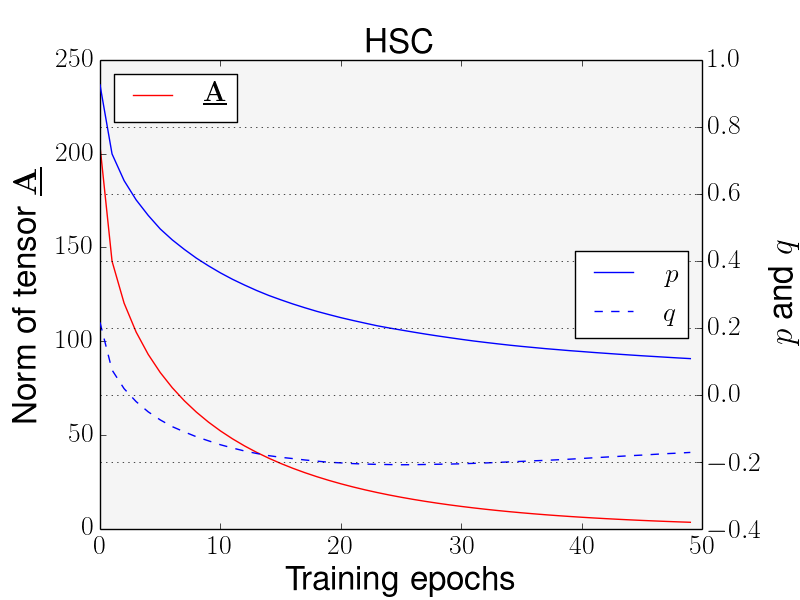}
   \end{subfigure}
\caption{The learnt model parameters for different word embeddings of 50 dimensions.}
\label{fig:Apq}
\end{figure}

\autoref{fig:Apq} shows the Frobenius norm of the tensor $\underline{\mat{A}}$ (on the left vertical axis) and the values of $p$ and $q$ (on the right vertical axis) for the six word embeddings. In all cases, we see that as the training progresses, $\underline{\mat{A}}$ goes to zero as predicted by \autoref{th:pd} under regularisation. 
Moreover, we see that approximately $p \approx -q = c$ is reached for some $c \in \R$ in all cases, which implies that $\mat{P}\approx -\mat{Q} = c\mat{I}$, which is the \PairDiff operator. 
Among the six input word embeddings compared in \autoref{fig:corr}, HSC has the highest mean correlation ($0.082$), which implies that its dimensions are correlated more than in the other word embeddings. 
This is to be expected by design because a hierarchical structure is imposed on the dimensions of the word embedding during training. 
However, HSC embeddings also satisfy the $\underline{\mat{A}} \approx \underline{\mat{0}}$ and $p \approx -q = c$ requirements, as expected by the \PairDiff.
This result shows that the claim of \autoref{th:pd} is empirically true even when the uncorrelation assumption is mildly violated.

\begin{figure}[t]
  \includegraphics[width=75mm]{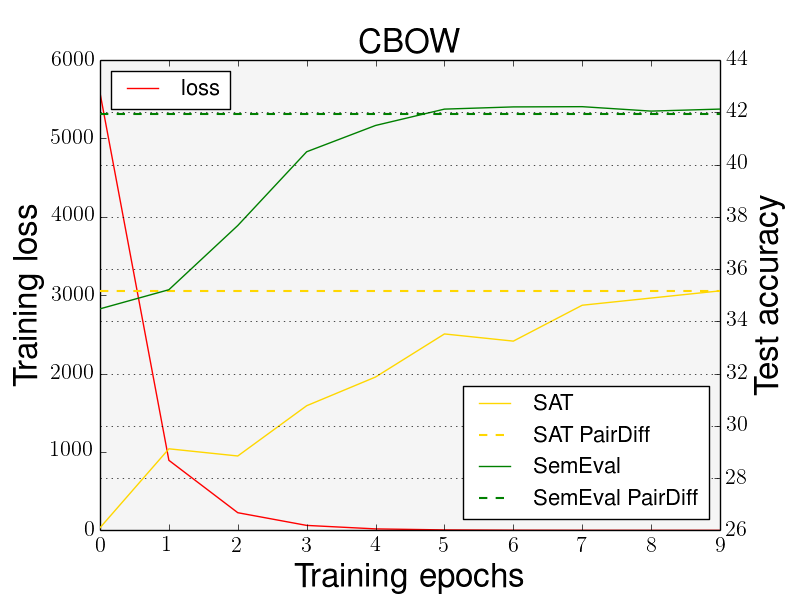}
  \caption{The training loss and test performance on SAT and SemEval benchmarks for relational embeddings. When the loss converges to zero, the performance on both benchmarks reaches that of the \PairDiff.}
  \label{fig:pd}
\end{figure}


\subsection{Generalisation Performance on Analogy Detection}
\label{sec:perform}

So far we have seen that the bilinear relational representation given by \eqref{eq:rel} does indeed converge to the form predicted by our theoretical analysis for different types of word embeddings. 
However, it remains unclear whether the parameters learnt from the training instances generated from the BATS dataset
 accurately generalise to other benchmark datasets for analogy detection.
To emphasise, our focus here is not to outperform relational representation methods proposed in previous works, but rather to empirically show that the learnt  operator converges to the popular \PairDiff for the analogy detection task.

To measure the generalisation capability of the learnt relational embeddings from BATS, 
we measure their performance on two other benchmark datasets: 
SAT~\cite{turney2003combining} and SemEval 2012-Task2\footnote{\url{https://sites.google.com/site/semeval2012task2/}}. 
Note that we \emph{do not} retrain $\underline{\mat{A}}$, $\mat{P}$ and $\mat{Q}$ in \eqref{eq:rel} on SAT nor SemEval, but simply to use
their values learnt from BATS because the purpose here to evaluate the generalisation of the learnt operator.

In SAT analogical questions, given a stem word-pair $(a,b)$ with five candidate word-pairs $(c,d)$, the task is to select the word-pair that is relationally  similar to the the stem word-pair. 
The relational similarity between two word-pairs $(a,b)$ and $(c,d)$ is computed by the cosine similarity between the corresponding relational embeddings $\vec{r}(\vec{a},\vec{b})$ and $\vec{r}(\vec{c},\vec{d})$.
The candidate word-pair that has the highest relational similarity with the stem word-pair is selected as the correct answer to a word analogy question. The reported accuracy is the ratio of the correctly answered questions to the total number of questions. 
On the other hand, SemEval dataset has 79 semantic relations, with each relation having ca. 41 word-pairs and four prototypical examples. 
The task is to assign a score for each word pair which is the average of the relational similarity between the given word-pair and prototypical word-pairs in a relation. Maximum difference scaling (MaxDiff) is used as the evaluation measure in this task.

\autoref{fig:pd} shows the performance of the relational embeddings composed from 50-dimensional CBOW embeddings.\footnote{Similar trends were observed for all six word embedding types but not shown here due to space limitations.}
The level of performance reported by \PairDiff on SAT and SemEval datasets are respectively $35.16\%$ and $41.94\%$, and are shown by horizontal dashed lines.
From \autoref{fig:pd}, we see that the training loss gradually decreases with the number of training epochs and the performance of the relational embeddings on SAT and SemEval datasets reach that of the \PairDiff operator.
This result indicates that the relational embeddings learnt not only converge to \PairDiff operator on training data but also generalise to unseen relation types in SAT and SemEval test datasets.

\section{Conclusion}

We showed that, if the word embeddings are standardised and uncorrelated, then the expected $\ell_{2}$ distance between analogous and non-analogous word-pairs is independent of bilinear terms, and the relation embedding further simplifies to the popular \PairDiff operator under regularised settings. Moreover, we provided empirical evidence showing the uncorrelation in word embedding dimensions, where their cross-dimensional correlations are narrowly distributed around a mean close to zero. 
An interesting future research direction of this work is to extend the theoretical analysis to nonlinear relation composition operators, such as for nonlinear neural networks.

\bibliography{Infer}
\bibliographystyle{acl_natbib}

\end{document}